\documentclass[letterpaper, 10 pt, conference, twoside]{ieeetran}

 \usepackage{graphicx}
\usepackage{amsmath}
\usepackage{amssymb}  
\usepackage{amsthm}

\usepackage{relsize}
\newtheorem{theorem}{Theorem}

\newtheorem{assumption}{Assumption}
\usepackage{placeins} 
\usepackage{subcaption} 
\usepackage{algorithm}
\usepackage{algorithmic}
\newtheorem{remark}{Remark}
\usepackage{xcolor}

\IEEEoverridecommandlockouts                              

\usepackage{graphbox}
\usepackage{fancyhdr}
\makeatletter
\def\ps@IEEEtitlepagestyle{%
  \def\@oddhead{\mycopyrightnotice}%
  \def\@oddfoot{\hbox{}\@IEEEheaderstyle\leftmark\hfil\thepage}\relax
  \def\@evenhead{\@IEEEheaderstyle\thepage\hfil\leftmark\hbox{}}\relax
  \def\@evenfoot{}%
}
\def\mycopyrightnotice{%
  \begin{minipage}{\textwidth}
  \centering \scriptsize
  This article has been accepted for publication in the IEEE/RSJ International Conference on Intelligent Robots and Systems (IROS) 2025. Copyright~\copyright~20XX IEEE.  Personal use of this material is permitted.  Permission from IEEE must be obtained for all other uses, in any current or future media, including reprinting/republishing this material for advertising or promotional purposes, creating new collective works, for resale or redistribution to servers or lists, or reuse of any copyrighted component of this work in other works.
  \end{minipage}
}
\makeatother

\title{\LARGE \bf
Distributed Fault-Tolerant Multi-Robot Cooperative Localization in Adversarial Environments}

\author{Tohid Kargar Tasooji \and 
Ramviyas Parasuraman
\thanks{School of Computing, University of Georgia, Athens, GA 30602 USA.                       

Author E-mails:
Tohid.KargarTasooji@uga.edu; ramviyas@uga.edu.}
\thanks{This research is supported by the Army Research Laboratory and was accomplished under DCIST Cooperative Agreement W911NF-17-2-0181. }
}

\begin{document}

\maketitle

\begin{abstract}

In multi-robot systems (MRS), cooperative localization is a crucial task for enhancing system robustness and scalability, especially in GPS-denied or communication-limited environments. However, adversarial attacks, such as sensor manipulation, and communication jamming, pose significant challenges to the performance of traditional localization methods. In this paper, we propose a novel distributed fault-tolerant cooperative localization framework to enhance resilience against sensor and communication disruptions in adversarial environments. We introduce an adaptive event-triggered communication strategy that dynamically adjusts communication thresholds based on real-time sensing and communication quality. This strategy ensures optimal performance even in the presence of sensor degradation or communication failure. Furthermore, we conduct a rigorous analysis of the convergence and stability properties of the proposed algorithm, demonstrating its resilience against bounded adversarial zones and maintaining accurate state estimation. Robotarium-based experiment results show that our proposed algorithm significantly outperforms traditional methods in terms of localization accuracy and communication efficiency, particularly in adversarial settings. Our approach offers improved scalability, reliability, and fault tolerance for MRS, making it suitable for large-scale deployments in real-world, challenging environments.
 
\end{abstract}
\section{INTRODUCTION}

Cooperative teams of multiple robots greatly enhance system robustness and scalability, making them crucial for applications such as area surveillance, assisted navigation, and disaster management \cite{1,2,3,4}. Among the core tasks for multi-robot systems (MRS), localization stands out as fundamental. Unlike independent localization for each robot, cooperative localization (CL) leverages shared information to estimate a robot's state by combining its own observations with data from other robots \cite{5, 6, 28}. Each robot gathers self-motion and relative measurements—such as range, bearing, or relative pose—using a combination of proprioceptive and exteroceptive sensors \cite{7, 8}. While these relative observations may be affected by sensor noise or environmental factors, they still provide valuable data to mitigate localization errors caused by accumulated drift or sensor failures \cite{9}. Furthermore, CL significantly reduces localization errors in GPS-denied environments, making it a highly effective approach. By improving the localization accuracy of individual robots, CL enables cooperative teams to perform tasks more efficiently and reliably, even under challenging conditions \cite{10,yang2023hier}. 

Decentralized cooperative localization algorithms, such as extended Kalman filters (EKF) \cite{11, 12}, unscented Kalman Filter (UKF) \cite{33}, and covariance intersection (CI) \cite{13}, offer advantages over centralized methods by providing better scalability and robustness, particularly in dynamic and noisy environments \cite{14}. The collaborative nature of robots also exposes them to adversarial attacks, which can compromise sensor data and communication, potentially resulting in system-wide failures. To address these vulnerabilities, it is essential to develop advanced decentralized algorithms for cooperative localization coupled with fault-tolerant strategies. These measures enhance resilience against malicious attacks and sensor malfunctions, ensuring safe, secure, and efficient functioning in real-world scenarios \cite{10, 15, 16}.

Communication is essential for localization in MRS, but it introduces challenges such as increased bandwidth demands, network congestion, packet loss, latency, and excessive power consumption \cite{parasuraman2018kalman,latif2023com}. As the number of robots grows, maintaining real-time coordination and accurate relative positioning becomes difficult \cite{bezzo2024,22}. Event-triggered communication mechanisms, which update only when significant changes occur, can mitigate these issues by reducing transmission frequency, alleviating network load, and conserving power. However, traditional event-triggered schemes with static thresholds may not adapt well to dynamic conditions such as adversarial environments.

In efforts to remedy these gaps, we propose a novel framework for multi-robot cooperative localization in adversarial and dynamic environments, focusing on resilient state estimation and adaptive event-triggered communication strategy. The contributions of this work are as follows:

\begin{itemize}
    \item First, we propose a fault-tolerant cooperative localization algorithm with adaptive event-triggered communication to enhance multi-robot resilience in adversarial environments. The strategy dynamically adjusts communication thresholds based on sensing reliability and signal quality, increasing updates in sensor-degraded zones and lowering thresholds in disrupted communication zones. This adaptive approach ensures real-time recovery from sensor or communication failures, improving robustness. Leveraging the robustness of Cubature Kalman Filter (CKF) \cite{34}, we develop a decentralized localization strategy with event-triggered communication. The integrated framework enhances scalability, reliability, and fault tolerance for large-scale MRS.

\item Second, we present a rigorous analysis of the convergence and stability properties of the proposed adaptive event-triggered cooperative localization algorithm. The analysis demonstrates that, under a bounded adversarial zone, our proposed localization remains bounded and converges to accurate position estimates.

\item Finally, we conduct a series of experiments in the Robotarium simulation-hardware platform \cite{Robotarium} to compare the performance of our algorithm with traditional cooperative localization methods. The results show that our algorithm outperforms these methods in terms of localization accuracy, communication efficiency, and scalability, particularly in adversarial environments.

\end{itemize}

\section{Related work}
The resilience and security of MRS in adversarial environments have become critical areas of research due to the increasing complexity and significance of these systems in real-world applications. Existing work explores various facets of this challenge, addressing specific aspects of resilience, security, and cooperative operation.

Mitra et al. \cite{18} investigate distributed state estimation in adversarial settings, targeting scenarios where deceptive sensing information misleads robots but overlook cases involving sensor deactivation or failure.
Ramachandran et al. \cite{19} emphasize adaptive resource reconfiguration to restore sensing capabilities after failures. While their work demonstrates adaptability, it does not address the need for preemptive resilience against adversarial threats. In another study, Wehbe et al. \cite{10} adopt a probabilistic approach to model uncertainties in robot interactions and estimate the likelihood of withstanding attacks. Their framework leverages binary decision diagrams (BDDs) and optimization techniques to adapt to system changes and is validated through simulations.

Focusing on adversarial environments, Zhou et al. \cite{20} propose a robust algorithm to optimize target-tracking performance in the presence of worst-case sensor and communication attacks, demonstrating strong robustness and efficient execution in simulations. They further extend this work in \cite{21} by introducing a resilient multi-target tracking algorithm that withstands any number of failures while maintaining scalability and provable performance bounds.

Recent efforts have also focused on robust state estimation under adversarial conditions to maintain localization accuracy. Tasooji et al. \cite{16} propose a secure decentralized event-triggered cooperative localization framework to mitigate communication attacks, including Denial of Service (DoS) and False Data Injection (FDI). Their algorithm ensures convergence under bounded attack rates. Bianchin et al. \cite{22} address adversarial localization and trajectory planning, proposing robust control designs and waypoint selection to counteract sensor spoofing and false control inputs.

Decentralized control architectures, which rely on local measurements for decision-making, have also gained attention. Cavorsi et al. \cite{23} develop a resilient path-planning algorithm using Control Barrier Functions (CBFs) to maintain safety while handling up to 
F adversaries. Catellani et al. \cite{24} propose a probabilistic framework employing particle filters for position estimation in communication-denied scenarios, while Park et al. \cite{25} introduce fault-tolerant distributed control for robot rendezvous, ensuring convergence despite faults but not explicitly addressing adversarial disruptions.

Addressing cyberattacks in cooperative localization remains an underexplored domain. Michieletto et al. \cite{26} present a distributed method to detect GNSS spoofing attacks in UAV formations, utilizing cascaded estimation algorithms for simultaneous localization and spoofing detection. Their approach integrates information-theoretic tools and measurement reliability-based decision logic, offering an innovative step toward secure localization in adversarial environments.

These studies collectively highlight the need for comprehensive solutions that integrate resilience, secure localization, and adaptive control to enable reliable MRS operation in adversarial and resource-constrained settings. We depart from the literature in following key ways: 
\begin{itemize}
    \item Most prior works model adversary strategies as random or periodic attacks. In contrast, our work considers a more practical and realistic scenario. We adopt the concept of a sensing danger zone and a communication danger zone, where the attack rate is higher near the center and decreases with distance from the center.
    
    \item Unlike existing solutions \cite{15, 16}, which typically rely on fixed communication protocols, our work introduces an adaptive event-triggered communication strategy. This strategy adjusts communication thresholds dynamically based on real-time sensing reliability and communication quality. 
    
    \item Additionally, works such as in \cite{22, 26} assume each robot is equipped with a GPS sensor for absolute position measurements; our approach is designed for real-world scenarios where GPS is unavailable, enhancing the practicality and usability of MRS in such conditions. 
    
    \item Furthermore, unlike \cite{15, 16}, which utilize the EKF for cooperative localization in adversarial environments, our work presents a more robust solution leveraging the CKF framework \cite{34} tailored to enhance cooperative localization under such challenging conditions.    
\end{itemize}
{\color{black} To the best of our knowledge, this is the first work that uniquely integrates spatially realistic danger zone modeling with an adaptive communication strategy, departing from prior works that rely on simplified adversarial models and static communication policies. These innovations enable more practical and resilient localization in GPS-denied and adversarial environments, as demonstrated by our experimental results.}

\section{Distributed Cooperative Localization in Adversarial Environments}
\subsection{Framwork Overview}

Our objective is to enhance the resilience of cooperative localization in MRS operating under adversarial or degraded conditions. Localization accuracy can be compromised by sensor degradation, communication disruptions, or adversarial attacks. To address these challenges, the framework integrates a sensing layer, communication layer, and estimation \& detection layer, incorporating robust attack detection and adaptive event-triggered communication.
Fig.~\ref{fig:overview} provides an architectural overview of the proposed framework.

The \textit{sensing layer} provides motion states and relative measurements using inertial and range/bearing sensors. However, sensor performance can be degraded due to noise, environmental interference, or adversarial attacks. The communication layer enables information exchange between robots but is susceptible to jamming, packet loss, and network congestion, affecting localization accuracy.

The \textit{estimation \& detection} layer, leveraging a Cubature Kalman Filter (CKF), ensures robust state estimation in nonlinear motion models. Two detection mechanisms enhance reliability: (1) an innovation residual-based attack detector that flags anomalies in relative measurements, and (2) an event detector that identifies abrupt state deviations due to adversarial activities or environmental changes.

The framework introduces sensing danger zones and communication danger zones to dynamically adapt communication based on real-time conditions. In sensing danger zones, communication frequency increases to compensate for degraded sensor measurements. In communication danger zones, transmission is reduced to mitigate network congestion while preserving localization accuracy.

The proposed CKF-based framework enhances resilience against sensor failures, communication disruptions, and adversarial zones by integrating attack detection with an adaptive event-triggered communication strategy. This approach significantly improves the fault tolerance and reliability of multi-robot cooperative localization, ensuring robust performance in real-world applications.

\begin{figure}[t]
    \centering
    \includegraphics[width=\linewidth]{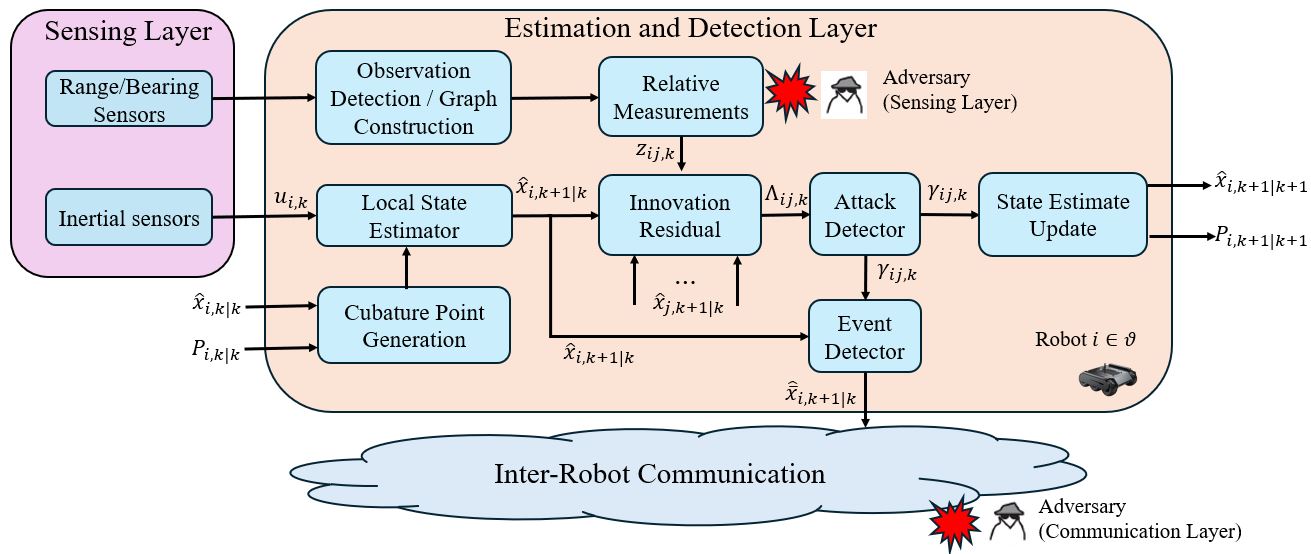}
    \caption{Block diagram of resilient multi-robot cooperative localization using CKF with adaptive event-triggered communication strategy in adversarial environments}
    \label{fig:overview}
\end{figure}

\subsection{Modeling of Agents and Adversaries}
\subsubsection{\textbf{Robot Motion Dynamics}}

Let us consider \( N \) robots, represented by the set \( V = \{1, \dots, N\} \). Each robot \( i \in V \) follows a discrete-time motion model given by the equation:
\begin{equation} \label{1}  
x_{i,k+1} = f_i(x_{i,k}, u_{i,k}),
\end{equation}  
where \( x_i \) is the state vector of robot \( i \), and \( u_i \) denotes the control input. The state of robot \( i \) is represented by the vector \( x_i = [x_i, y_i, \theta_i]^T \), where \( x_i \) and \( y_i \) are the coordinates of the robot's position in a 2D plane, and \( \theta_i \) is its heading angle. The control input for robot \( i \) is given by \( u_i = [v_i, \omega_i]^T \), where \( v_i \) represents the linear velocity and \( \omega_i \) represents the angular velocity. 

The unicycle motion model for each robot is given by:
\begin{equation} \label{2}  
\begin{aligned}
\left\{  
\begin{aligned}  
x_{i,k+1} &= x_{i,k} + v_{i,k} \cos(\theta_{i,k}) \Delta t, \\
y_{i,k+1} &= y_{i,k} + v_{i,k} \sin(\theta_{i,k}) \Delta t, \\
\theta_{i,k+1} &= \theta_{i,k} + \omega_{i,k} \Delta t.  
\end{aligned}  
\right.  
\end{aligned}
\end{equation}

\begin{table}[t]
\caption{Definition of notations used in this paper.}
\label{tab:definition}
\centering
\scriptsize
\begin{tabular}{|c|l|}
\hline
\textbf{Notation} & \textbf{Description} \\
\hline
$\hat{x}_{k|k-1}$ & Predicted state estimate at time step $k$ \\
$P_{k|k-1}$ & Predicted error covariance at time step $k$ \\
$S_{k|k-1}$ & Cholesky factor of the predicted error covariance \\
$X_{i,k|k-1}$ & Cubature points at time step $k-1$ \\
$f(X, u)$ & State transition function \\
$Q_{k-1}$ & Process noise covariance at time step $k-1$ \\
$Z_{ij,k|k-1}$ & Relative measurement between robots $i$ and $j$ \\
$\hat{z}_{ij,k|k-1}$ & Predicted measurement between robots $i$ and $j$ \\
$\Lambda_{ij,k}$ & Innovation for the measurement between robots $i$ and $j$ \\
$r_{ij,k}$ & Residual for attack detection \\
$\delta_{ij}(G)$ & Threshold for event-triggered communication \\
$\alpha$ & Constant for event-triggered threshold \\
$\lambda_2(G)$ & Second smallest eigenvalue of the graph's Laplacian matrix \\
$\gamma$ & Constant for event-triggered threshold  \\
$\zeta_s$ & Weight for sensing danger zone in event-triggered condition\\
$\zeta_c$ &  Weight for communication danger zone in event-triggered condition \\
$D_s(i)$ & Adversary rate for sensing danger zone in robot $i$ \\
$D_c(i)$ & Adversary rate for communication danger zone in robot $i$ \\
$R_k$ & Measurement noise covariance at time step $k$ \\
$P_{ij,k|k-1}^{zz}$ & Innovation covariance between robots $i$ and $j$ \\
$P_{ij,k|k-1}^{xz}$ & Cross-covariance between state and innovation \\
$K_{ij,k}$ & Kalman gain for measurement update between robots $i$ and $j$ \\
$\hat{x}_{i,k|k}$ & Updated state estimate for robot $i$ \\
$P_{i,k|k}$ & Updated error covariance for robot $i$ \\
$a_{ij,k}$ & Indicator of interaction between robots $i$ and $j$ \\
$\lambda_{ij,k}$ & Indicator of transmission between robots $i$ and $j$ \\
\hline
\end{tabular}
\end{table}

In this collaborative framework, each robot uses its onboard sensors to estimate its own position and orientation (self-localization) and neighbor robot information (range and bearing) and shares this information with neighboring robots. This collaborative localization process improves the accuracy of position estimates across the entire network, particularly in environments where GPS signals are unreliable or unavailable. By leveraging shared information from all robots, the system becomes more resilient and efficient in its operation.

\subsubsection{\textbf{Sensing Model}}

Each robot in the system is assumed to be equipped with a range and bearing sensor that measures relative distances and angles to nearby robots. The robots are organized in a connected sensing graph \( G = (V, E) \), where the edges \( E \) represent the communication links formed based on the sensor measurements. The edge set is defined as:
\[
E = \{(i, j) \mid j \in \mathcal{N}_i \Leftrightarrow |p_i - p_j| \leq R \},
\]
where \( R \) is the sensing range, and \( \mathcal{N}_i \) represents the set of neighbors for robot \( i \). The sensing graph reflects the direct measurements that each robot can make to its neighbors within the sensing range.

A robot \( i \in V \) observes a neighboring robot \( j \in \mathcal{N}_i \) using the following measurement model:
\begin{equation} \label{3}
z_{ij,k} = h_{ij}(x_{i,k}, x_{j,k}) + v_{ij,k}, \quad \forall i \in V, j \in \mathcal{N}_i,
\end{equation}
where \( z_{ij,k} \) represents the measurement obtained by robot \( i \) of robot \( j \), and \( v_{ij,k} \) is the measurement noise, assumed to be zero-mean Gaussian noise with covariance \( R \), i.e., \( v_{ij,k} \sim \mathcal{N}(0, R) \). The function \( h_{ij}(x_{i,k}, x_{j,k}) \) models the relationship between the states of robots \( i \) and \( j \) as perceived by robot \( i \)'s sensor. In the case of range and bearing sensors, this function is typically a non-linear mapping that depends on the relative positions of the robots.

We assume that the noise in the measurements, \( v_{ij,k} \), is independent across different robot pairs, meaning that each robot's sensor noise does not interfere with the measurements from other robots. This model facilitates the estimation of relative positions and orientations within the robot network and helps improve the overall state estimation of the system as robots collaborate by sharing their sensory information.

\subsubsection{\textbf{Attack Model}}
This part presents an attack model for cooperative localization in MRS, focusing on Sensing Danger Zones (SDZs) and Communication Danger Zones (CDZs), which represent areas where sensing and communication performance degrade due to environmental or adversarial factors.

A \textit{Sensing Danger Zone} (SDZ) is an area where sensor reliability is compromised, either due to environmental factors (e.g., partial occlusion or physical blockage to the sensors) or adversarial interference. The \textit{Sensing Risk Field} $R_s(p)$ quantifies the failure probability at position $p$, and the SDZ is defined as the set of positions where $R_s(p) \geq \tau_s$:
\begin{equation}
D_s = \{ p \mid R_s(p) \geq \tau_s \}
 \end{equation}
In the SDZ, adversaries can inject noise into sensor measurements, expressed as:
\begin{equation}
z_{ij,k} = h_{ij}(x_{i,k}, x_{j,k}) + \zeta_{ij,k} \Gamma_{ij,k} + \nu_{ij,k}
 \end{equation}
where $\zeta_{ij,k}$ indicates compromised measurements, and the adversary manipulates the residual error:
\begin{equation}
r_{ij,k} = z_{ij,k} - \hat{z}_{ij,k|k-1}
 \end{equation}
by maximizing $\|r_{ij,k}\|$ to induce large estimation errors while avoiding detection.

A \textit{Communication Danger Zone} (CDZ) is a region where communication is unreliable due to interference or jamming. The \textit{Communication Risk Field} $R_c(p)$ quantifies communication failure probability, and the CDZ is defined as:
\begin{equation}
D_c = \{ p \mid R_c(p) \geq \tau_c \}
 \end{equation}
In a CDZ, robots rely only on internal sensors and motion models, which leads to unchanged state estimates and increasing uncertainty. 
{\color{black}
Let $\sigma_{ij,k} \in \{0, 1\}$ be a binary communication indicator at time step $k$, where $\sigma_{ij,k} = 1$ denotes successful communication between robot $i$ and robot $j$, and $\sigma_{ij,k} = 0$ indicates a communication failure. The state update rule is given by:
\begin{equation}
\hat{x}_{i,k+1|k+1} = \begin{cases} 
    \hat{x}_{i,k+1|k}, & \text{if } \sigma_{ij,k} = 0 \\
    \hat{x}_{i,k+1|k+1}, & \text{if } \sigma_{ij,k} = 1 
\end{cases}
 \end{equation}
where $\hat{x}_{i,k+1|k+1}$ is the updated state using the Kalman gain $K_{ij,k}$ when communication is available, and remains unchanged otherwise.
}

\subsubsection{\textbf{Communication}}

We employ a decentralized communication framework for a network of robots, represented as an undirected graph \( G = \{V, E, A\} \), where \( |V| = N \) denotes the robots, and \( E \) defines communication links. The adjacency matrix \( A \in \mathbb{R}^{N \times N} \) captures connectivity, with \( a_{ij} = 1 \) if \( v_{ij} \in E \) and \( a_{ii} = 0 \). The Laplacian matrix \( L \in \mathbb{R}^{N \times N} \) is defined as \( l_{ii} = \sum_{j \neq i} a_{ij} \) and \( l_{ij} = -a_{ij} \) for \( i \neq j \). Robots communicate bidirectionally, enabling mutual detection and measurement of relative distances \( z_{ij} \).

To optimize communication efficiency, we propose an adaptive event-triggered communication scheme where robot \( i \) transmits to \( j \) only when the innovation  
\begin{equation}
\Lambda_{ij,k} = z_{ij,k} - \hat{z}_{ij,k|k-1}
\end{equation}
exceeds a dynamic threshold \( \delta_{ij}(G) \), ensuring efficient bandwidth utilization. The predicted measurement \( \hat{z}_{ij,k|k-1} \) is derived from the predicted states \( \hat{x}_{i,k|k-1} \) and \( \hat{x}_{j,k|k-1} \). Therefore, communication is triggered if: 
\begin{equation}
\|\Lambda_{ij,k}\| > \delta_{ij}(G).
\end{equation}

The threshold \( \delta_{ij}(G) \) dynamically adapts based on network connectivity, innovation magnitude, and environmental risk zones:  
\begin{equation} \label{eq:event-threshold}
\delta_{ij}(G) = \frac{\alpha}{\lambda_2(G)} \left(1 + \gamma\,\|\Lambda_{ij,k}\|\right) \left(1 - \zeta_s D_s(i) - \zeta_c D_c(i)\right),
\end{equation}
where \( \lambda_2(G) \) represents the graph's algebraic connectivity, \(\alpha\) and \(\gamma\) are thresholds of event-triggered condition, and \( D_s(i), D_c(i) \) denote rate of adversary for Sensing and Communication Danger Zones, respectively. \( \zeta_s(i), \zeta_c(i) \) are their corresponding weights. This framework optimizes communication efficiency while preserving estimation accuracy in dynamic and potentially adversarial environments.
\begin{remark}
 The event-triggering rule based on the adaptive threshold in \eqref{eq:event-threshold} balances communication efficiency and estimation accuracy by adapting the threshold based on three main factors: network connectivity, innovation magnitude, and environmental risk. In well-connected networks (i.e., higher \( \lambda_2(G) \) ), the threshold increases to suppress redundant communications, thereby conserving bandwidth and energy. Conversely, when the innovation is significant or when robots operate in high-risk zones, the threshold is reduced to ensure frequent information exchange. This adaptive mechanism not only mitigates communication overload under normal conditions but also robustly responds to uncertainties and adversarial challenges, making it particularly suitable for dynamic and adversarial environments.   
\end{remark}

\subsection{Integration with CKF Framework} 
We propose a robust cooperative localization framework for MRS using the CKF, optimized for non-smooth maneuvers, limited sensing, dynamic network topologies, and adversarial conditions. The CKF outperforms in handling nonlinear dynamics and measurement models, ensuring reliable state estimation even in challenging environments. Unlike the EKF and UKF, the CKF uses cubature points to approximate Gaussian-weighted integrals, avoiding linearization. It improves upon the UKF with a third-degree spherical-radial cubature rule for accurate state estimation in high-dimensional systems \cite{34}. In adversarial scenarios like jamming or spoofing, the event-based CKF remains resilient by filtering out erroneous data and employing robust estimation techniques.
The event-based CKF cooperative localization framework for both normal and adversarial conditions is presented in Algorithm 1.

\begin{algorithm}
\scriptsize
\caption{Distributed Fault-Tolerant Multi-Robot Cooperative Localization}
\begin{algorithmic}[1]
\STATE \textbf{Initialize:} State estimate $\hat{x}_{0|0}$, covariance $P_{0|0}$, graph $G$, thresholds $\tau_s, \tau_c, \rho$.
\FOR{each time step $k$}
    \STATE \textbf{Prediction:}
    \STATE Compute Cholesky factor: $P_{i,k-1|k-1} = S_{i,k-1|k-1} S_{i,k-1|k-1}^\top$.
    \STATE Generate cubature points: $X_{i,k-1|k-1} = S_{k-1|k-1} \xi_i + \hat{x}_{k-1|k-1}$.
    \STATE Propagate: $X_{i,k|k-1}^* = f(X_{i,k-1|k-1}, u_{k-1})$.
    \STATE Compute predicted state: $\hat{x}_{k|k-1} = \frac{1}{m} \sum_{i} X_{i,k|k-1}^*$.
    \STATE Compute covariance: $P_{i,k|k-1} = \text{Var}(X_{i,k|k-1}^*) + Q_{i,k-1}$.
    
    \STATE \textbf{Correction:}
    \STATE Compute Cholesky factor: $P_{i,k|k-1} = S_{i,k|k-1} S_{i,k|k-1}^\top$.
    \STATE Generate cubature points: $X_{i,k|k-1} = S_{i,k|k-1} \xi_i + \hat{x}_{i, k|k-1}$.
    \FOR{each neighbor $j \in S_i$}
        \STATE Compute measurement: $Z_{ij,k|k-1} = h_{ij}(X_{i,k|k-1}, X_{j,k|k-1})$.
        \STATE Compute innovation: $\Lambda_{ij,k} = z_{ij,k} - \hat{z}_{ij,k|k-1}$.
        
        \STATE \textbf{Attack Detection:}
        \IF{$\| \Lambda_{ij,k} \| > \rho$} 
            \STATE Discard $z_{ij,k}$.
        \ENDIF
        
        \STATE \textbf{Event-Triggered Communication:}
        \STATE Compute threshold:
        \[
        \delta_{ij}(G) = \frac{\alpha}{\lambda_2(G)} (1 + \gamma \| \Lambda_{ij,k} \|) (1 - \zeta_s D_s(i) - \zeta_c D_c(i)).
        \]
        \IF{$\| \Lambda_{ij,k} \| > \delta_{ij}(G)$}
            \STATE Transmit measurement.
            \STATE Compute Kalman gain: $K_{ij,k} = P_{ij,k|k-1}^{xz} (P_{ij,k|k-1}^{zz})^{-1}$.
            \STATE Update state: $\hat{x}_{i,k|k} = \hat{x}_{i,k|k-1} + \sum_{j \in S_i} a_{ij,k} \lambda_{ij,k} K_{ij,k} \Lambda_{ij,k}$
            \STATE Update covariance:$
            P_{i,k|k} = P_{i,k|k-1} - K_{ij,k} P_{ij,k|k-1}^{zz} K_{ij,k}^T.
            $
        \ENDIF
    \ENDFOR
\ENDFOR
\end{algorithmic}
\end{algorithm}

\subsection{Theoretical Analysis}
We establish key assumptions about system dynamics, measurement functions, noise, and adversarial attacks. The analysis aims to show that the CKF's estimation error converges to a bounded value over time, even in adversarial environments. By leveraging Lipschitz continuity in the state dynamics and measurement functions, we demonstrate that the event-triggered communication mechanism mitigates noise and adversarial interference, ensuring robust localization performance.
\begin{assumption} \label{Lipschitz Continuous Dynamics}
The state dynamics for each robot $i$ are governed by a Lipschitz continuous function $f$:
\begin{equation}
\|f(x_{i,k}, u_k) - f(x_{i,k}', u_k)\| \leq L_f \|x_{i,k} - x_{i,k}'\|
\end{equation}
where $L_f$ is the Lipschitz constant.
\end{assumption}
\begin{assumption} \label{Lipschitz Continuous Measurement Function}
 The measurement function $h(x_{i,k}, x_{j,k})$ is Lipschitz continuous:
\begin{equation}
\small
\|h(x_{i,k}, x_{j,k}) - h(x_{i,k}', x_{j,k}')\| \leq L_h \left(\|x_{i,k} - x_{i,k}'\|   + \|x_{j,k} - x_{j,k}'\|\right)
\end{equation}
where $L_h$ is the Lipschitz constant.
\end{assumption}
\begin{assumption} \label{Bounded Process and Measurement Noise}
  The process noise $w_k$ and measurement noise $v_k$ are bounded:
\begin{equation}
\|Q_k\| \leq \sigma_Q, \quad \|R_k\| \leq \eta_R
\end{equation}
\end{assumption}
\begin{assumption} \label{Bounded Adversarial Attacks}
The adversary exploits SDZ and CDZ, where environmental or adversarial factors degrade system performance. The number of attacks per time step on sensing and communication is bounded by $\alpha_s$ and $\alpha_c$.
\end{assumption}
\begin{theorem} \label{Theorem}
Consider the nonlinear dynamic system (1)-(3) along with the event-triggered mechanism (11) and assume that assumptions (1)-(4) are satisfied. The state estimation error of the distributed event-based Cubature Kalman Filter (CKF) for multi-robot cooperative localization, operating in dynamic network topology with adversarial environments, converges to a bounded value as time $k \to \infty$.
\end{theorem}

\begin{proof}
1) \textit{State Prediction Step}: The predicted state for robot $i$ is given by:
\begin{equation}
\hat{x}_{i,k+1|k} = f(\hat{x}_{i,k|k}, u_k) + w_{i,k}, \nonumber
\end{equation}
where $w_{i,k} \sim \mathcal{N}(0, Q_k)$. The prediction error is:
\begin{equation}
e_{i,k+1|k} = x_{i,k+1} - \hat{x}_{i,k+1|k} \nonumber
\end{equation}
Using the Lipschitz continuity of $f$:
\begin{equation}
\|e_{i,k+1|k}\| \leq L_f \|e_{i,k|k}\| + \|w_{i,k}\|. \nonumber
\end{equation}
Taking the expectation:
\begin{equation}
\mathbb{E}[\|e_{i,k+1|k}\|^2] \leq L_f^2 \mathbb{E}[\|e_{i,k|k}\|^2] + \sigma_Q. \nonumber
\end{equation}

2) \textit{Measurement Update Step}: The state estimation update incorporates measurements:
\begin{align}
\scalebox{0.8}{$
\begin{aligned}
\hat{x}_{i,k+1|k+1} = \hat{x}_{i,k+1|k}  + \mathlarger{\mathlarger{\sum}}_{j \in S_i} a_{ij,k+1} \lambda_{ij,k+1} \sigma_{ij,k} K_{ij,k+1}  \left( z_{i,k+1} - \hat{z}_{ij,k+1|k} \right) \nonumber
\end{aligned} 
$}
\end{align}
Considering the first-order expansion of $ h_{ij}(x_{i,k}, x_{j,k}) $ around $(\hat{x}_{i,k+1|k}, \hat{x}_{j,k+1|k})$ be
\begin{align}
\begin{aligned}
h_{ij}(x_{i,k+1}, x_{j,k+1}) \approx h_{ij}(\hat{x}_{i,k+1|k}, \hat{x}_{j,k+1|k})   +  H_{i, k+1}e_{i,k+1|k} \\ + H_{j, k+1}e_{j,k+1|k}
 + O(e_{i,k+1|k}, e_{j,k+1|k})
\end{aligned} \nonumber
\end{align}
where $O(e_{i,k+1|k}, e_{j,k+1|k})$ represents higher order terms. 
The estimation error after the update is:
\begin{align}
\scalebox{0.8}{$
\begin{aligned}
e_{i,k+1|k+1} = \left(I - \sum_{j \in S_i} a_{ij,k+1}\lambda_{ij,k+1}\sigma_{ij,k}K_{ij,k+1}H_{i,k+1}\right)e_{i,k+1|k} \\
- \sum_{j \in S_i} a_{ij,k+1}\lambda_{ij,k+1}\sigma_{ij,k}K_{ij,k}v_{ij,k} 
- \sum_{j \in S_i} a_{ij,k+1}\lambda_{ij,k+1}\sigma_{ij,k} K_{ij,k}e_{j,k+1|k} \\
- \sum_{j \in S_i} a_{ij,k+1}\lambda_{ij,k+1}\sigma_{ij,k}K_{ij,k}O(e_{i,k+1|k}, e_{j,k+1|k})\\
- \sum_{j \in S_i} a_{ij,k+1}\lambda_{ij,k+1}\sigma_{ij,k}K_{ij,k}\zeta_{ij,k} \Gamma_{ij,k} 
\end{aligned} \nonumber
$}
\end{align}

Expectations over the estimation error result in the expectation over its individual components. 

{\color{black}
Now we introduce adversarial attacks $\Delta (e_{k+1})$
as an additional error term. Under the assumption that the gains \(K_{ij,k}\) are optimal and that all the event-triggered coefficients and higher-order terms can be bounded appropriately, we can assert the existence of constants \(\mu \in (0,1)\) and \(\beta>0\) s.t.
\begin{equation}
\mathbb{E}\Big[\|e_{i,k+1|k+1}\|^2\Big] \le \mu\,\mathbb{E}\Big[\|e_{i,k+1|k}\|^2\Big] + \beta\,\eta_R + \Delta (e_{k+1}). \nonumber
\end{equation}

Here, \(\mu\) represents the impact due to the measurement update (including the effect of neighbor agents’ error terms and the linearization residuals), while \(\beta\,\eta_R\) bounds the contribution of the measurement noise and higher-order terms.

Iterating the inequality over time yields
\begin{align}
\begin{aligned}
\mathbb{E}\Big[\|e_{i,k+1|k+1}\|^2\Big] \le \mu^k\,\mathbb{E}\Big[\|e_{i,0}\|^2\Big]  \\  + \sum_{t=0}^{k}\mu^t\ \Big (\beta\,\eta_R +\Delta (e_{k+1}) \Big )
\end{aligned} \nonumber
\end{align}
Since the geometric series \(\sum_{t=0}^{\infty}\mu^t\) converges to \(\frac{1}{1-\mu}\) (for \(\mu\in (0,1)\)), we obtain
\begin{equation}
\lim_{k\to\infty}\mathbb{E}\Big[\|e_{i,k+1|k+1}\|^2\Big] \le \frac{\beta\,\eta_R + \Delta (e_{k+1})}{1-\mu}. \nonumber
\end{equation}}
Thus, the estimation error converges to a bounded value, demonstrating robustness under dynamic network topology and adversarial conditions.
\end{proof}
\begin{remark}
Theorem \ref{Theorem} establishes the boundedness of the state estimation error in the proposed event-triggered CKF framework for multi-robot localization under dynamic topology and adversarial environments. The error exhibits geometric decay, converging to an upper bound influenced by noise, adversarial parameters, and event-triggering thresholds. Notably, the additional adversarial term $\Delta (e_{k+1})$ is mitigated by the proposed event-triggered strategy, which adapts communication thresholds based on sensing reliability and signal quality. This adaptive mechanism enhances robustness, ensuring real-time recovery from sensor or communication failures while improving scalability and fault tolerance in large-scale robotic networks.
\end{remark}

\begin{figure}[t]
\captionsetup{justification=centering}
    \centering
    \begin{subfigure}[b]{0.23\textwidth}
        \includegraphics[width=\textwidth]{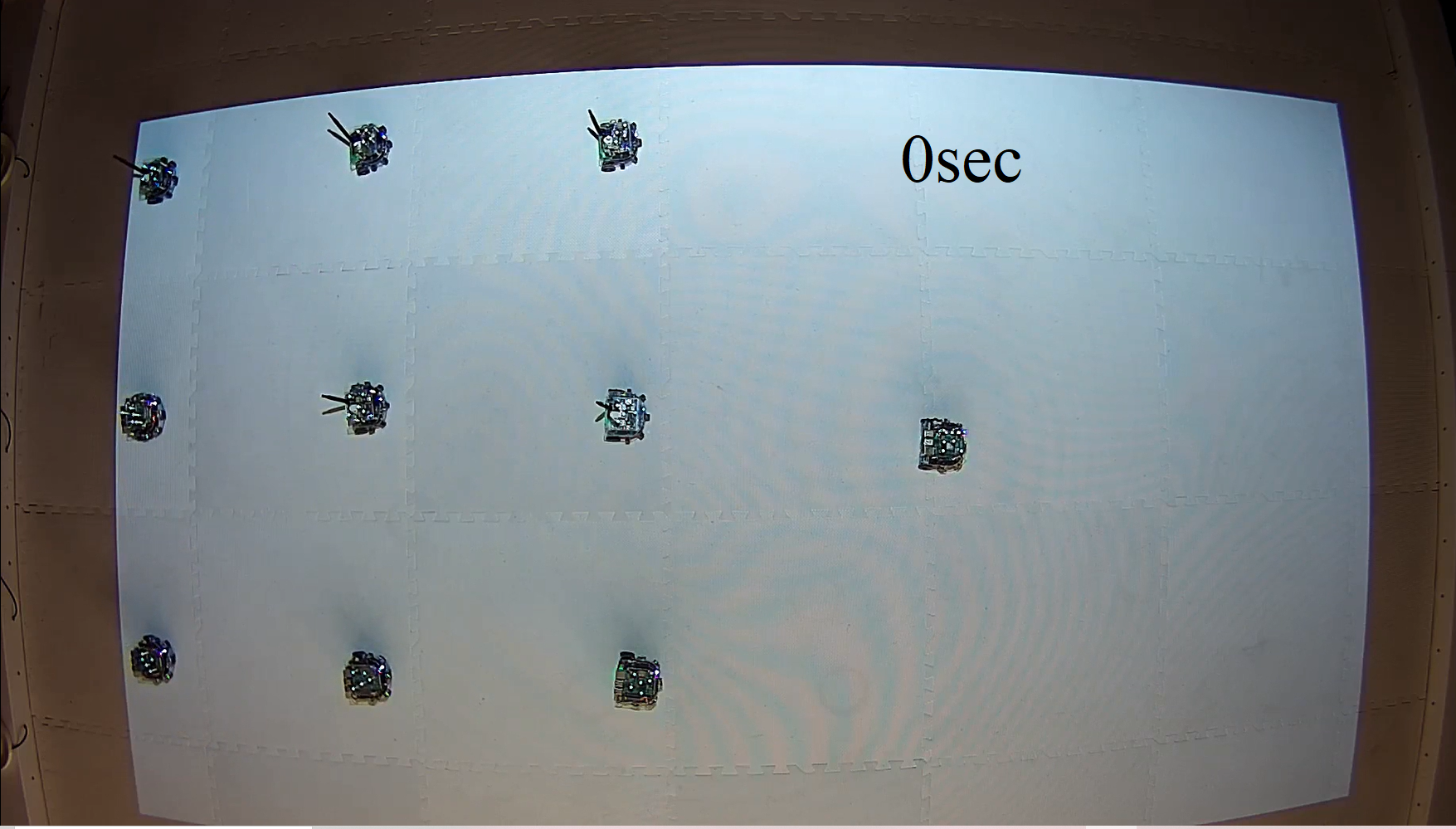}
    \end{subfigure}
    ~ 
        \begin{subfigure}[b]{0.23\textwidth}
        \includegraphics[width=\textwidth]{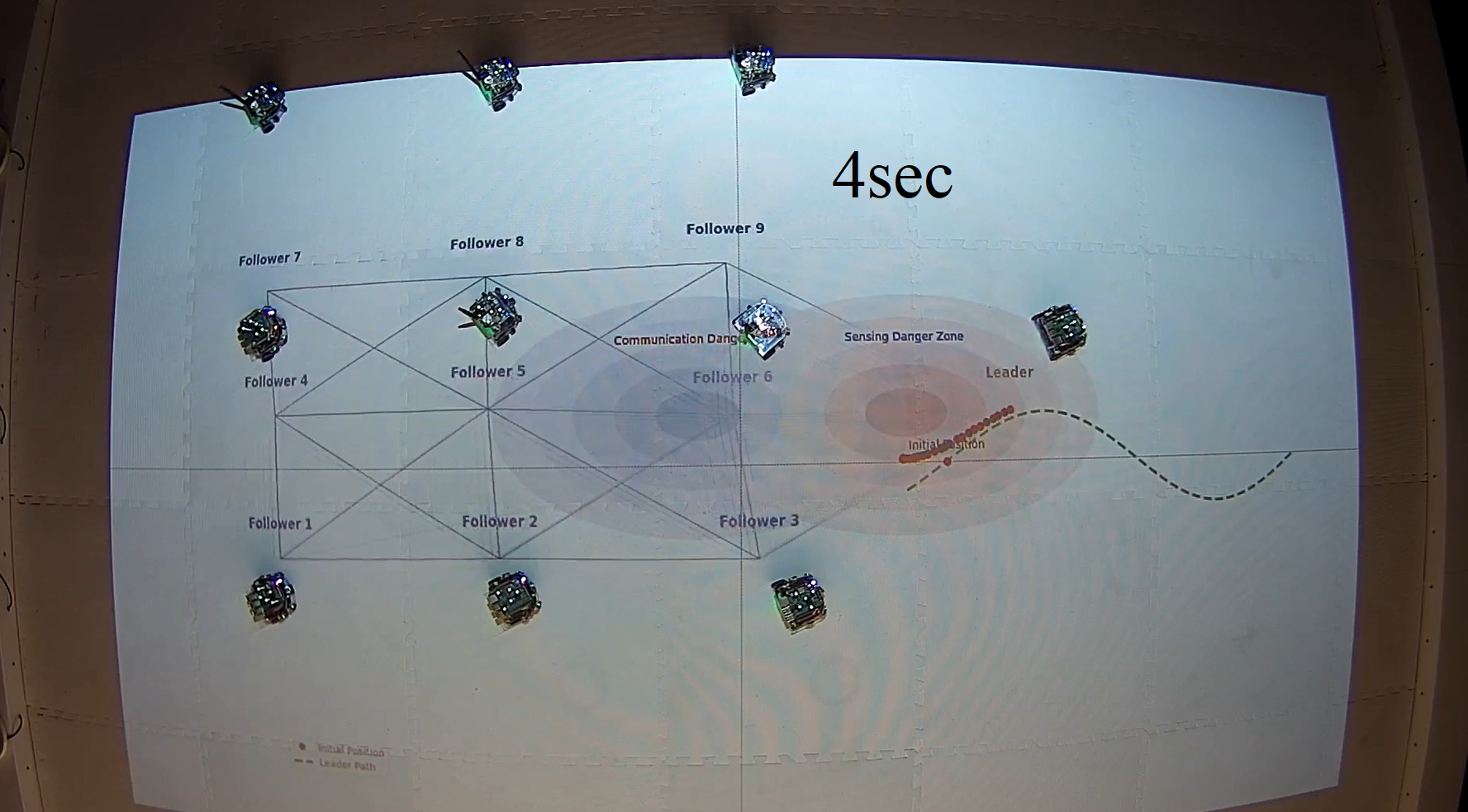}
    \end{subfigure}
    ~ 
     \begin{subfigure}[b]{0.23\textwidth}
        \includegraphics[width=\textwidth]{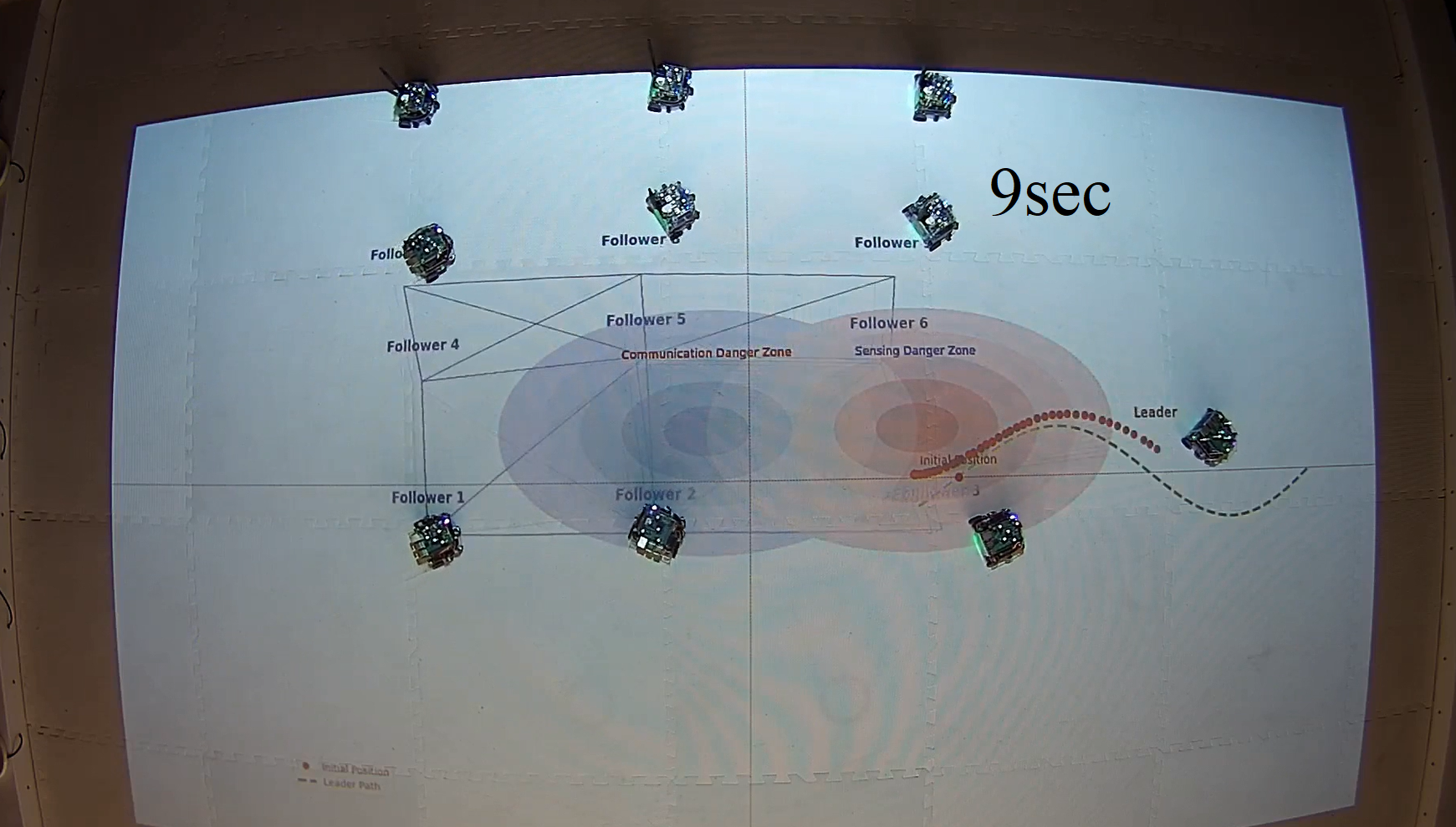}
    \end{subfigure}
    ~ 
     \begin{subfigure}[b]{0.23\textwidth}
        \includegraphics[width=\textwidth]{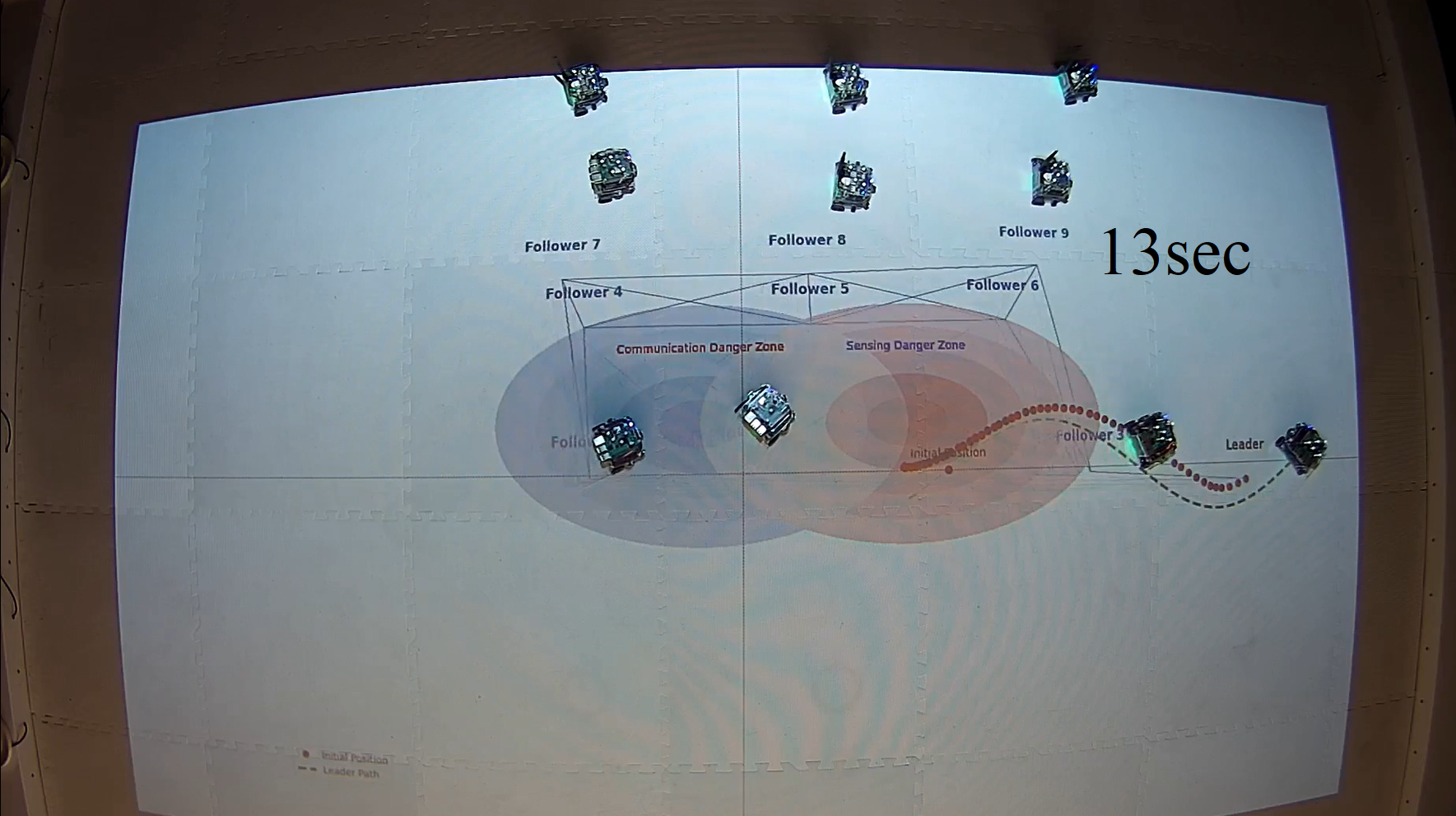}
    \end{subfigure}
    \caption{Illustration of experiments conducted in the Robotarium hardware platform \cite{Robotarium}, showing the evolution of robot trajectories over time.}
  \label{fig:robotarium}
\end{figure} 

\section{Experiments}
The experimental setup involved deploying 17 robots within the Robotarium simulation-hardware platform \cite{Robotarium}, which provides remote access to a multi-robot hardware testbed. The robots were initially arranged in a square lattice configuration, confined within a 3m×3m area of the testbed. We investigate a leader-follower scenario where the leader robot traverses predefined waypoints, and follower robots track its trajectory using desired positions communicated by the leader. Fig.~\ref{fig:robotarium} presents an illustration of the experiment setting.
In the experiments, the parameters for range and bearing noise are initialized to 0.005 and 0.05, respectively. The robots' communication and sensing radius is set to 1, and the network comprises 17 robots. The communication topology is dynamic, potentially changing based on the relative range between robots during movement.

For benchmarking, we compare our method against the distributed cooperative localization algorithms based on the EKF \cite{16} and the UKF \cite{33}. Furthermore, we evaluate the performance of our proposed algorithm under adversarial conditions, accounting for both SDZ and CDZ. 

To assess performance, we define the Mean Square Localization Error (MSLE) as: $MSLE = \frac{1}{N} \sum_{i=1}^{N} \sum_{\substack{j=1 \ j \neq i}}^{N} \left( | \mathbf{x}_i - \mathbf{x}_j | - | \hat{\mathbf{x}}_i - \hat{\mathbf{x}}_j | \right)^2$, where \(\mathbf{x}\) represents the true positions of the robots in the ground truth reference frame and \( \hat{\mathbf{x}}  \) denotes the estimated positions obtained from the distributed localization process.

\begin{figure}[t]
\captionsetup{justification=centering}
    \centering
    \begin{subfigure}[b]{0.45\linewidth}
        \includegraphics[width=\textwidth]{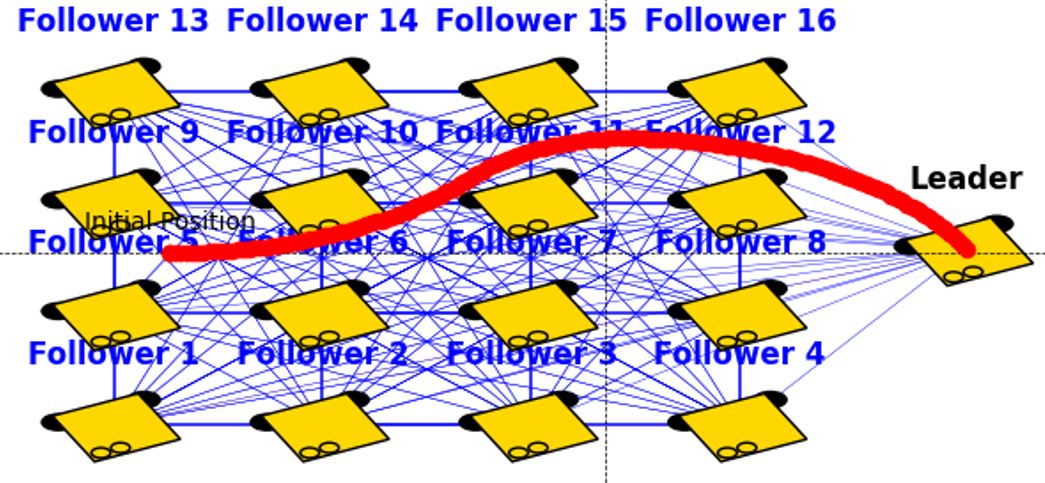}
        \caption{ Arc-Turn maneuver}
        \label{fig:gull-nodanger}
    \end{subfigure}
    ~ 
        \begin{subfigure}[b]{0.45\linewidth}
        \includegraphics[width=\textwidth]{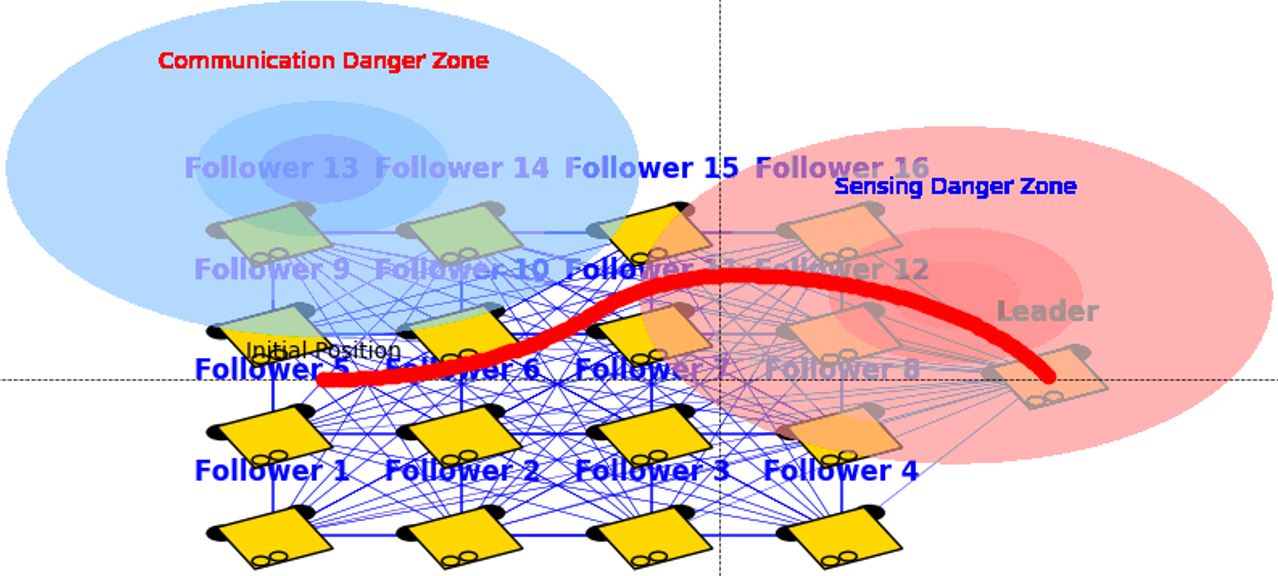}
        \caption{ Arc-Turn maneuver and both danger sensing and  communication zones}
        \label{fig:gull-danger}
    \end{subfigure}
    \caption{Depiction of non-smooth maneuvers and without/with adversarial settings.}
  \label{Fig:trajectory}
\end{figure} 

\begin{figure*}[t]
\captionsetup{justification=centering}
    \centering
    ~ 
    \begin{subfigure}[b]{0.31\linewidth}
        \includegraphics[width=\textwidth]{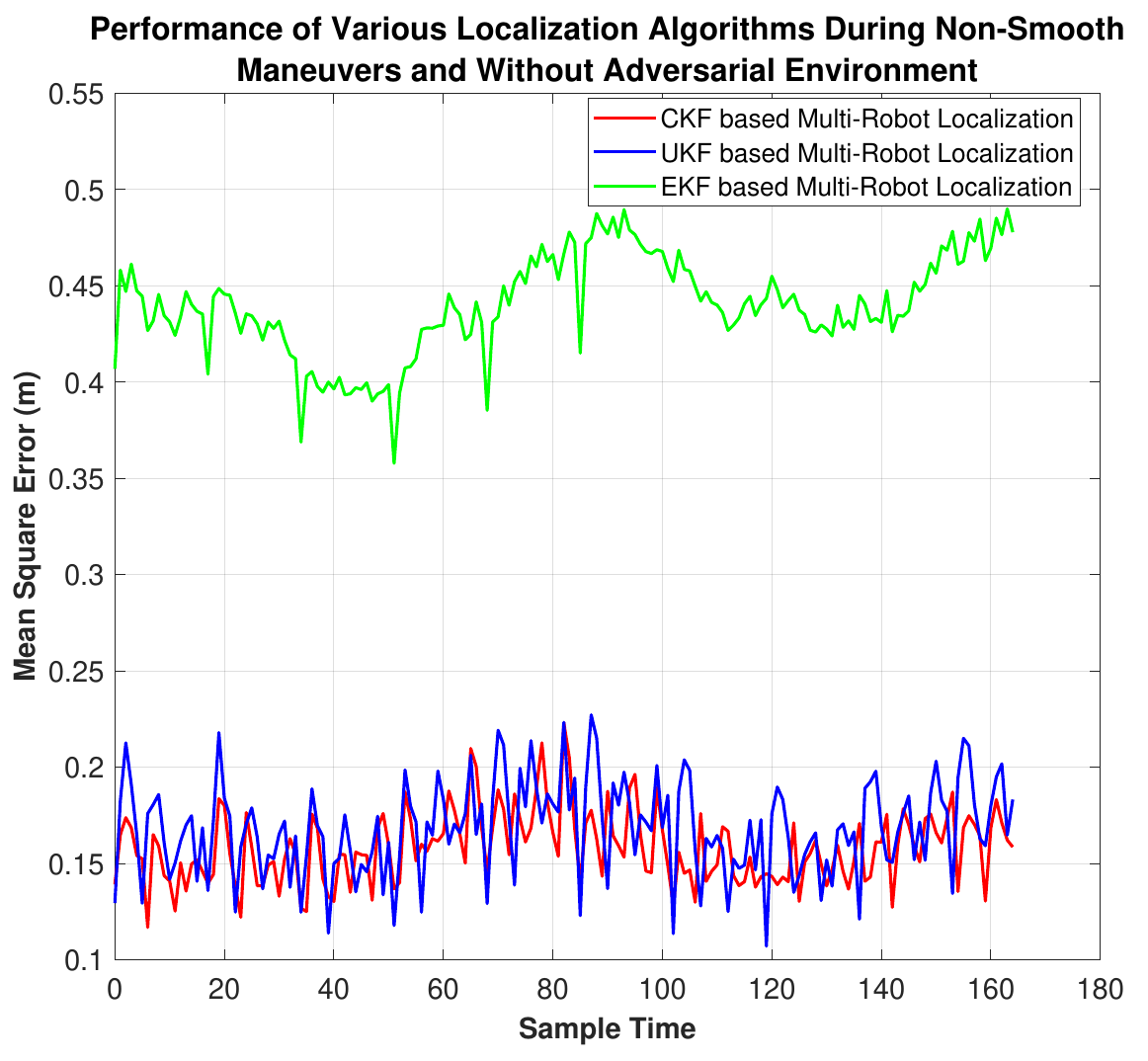}
        \caption{\color{black}Non-adversarial environment and Arc-turn maneuver}
        \label{fig:tiger1}
    \end{subfigure}
 ~ 
     \begin{subfigure}[b]{0.31\linewidth}
        \includegraphics[width=\textwidth]{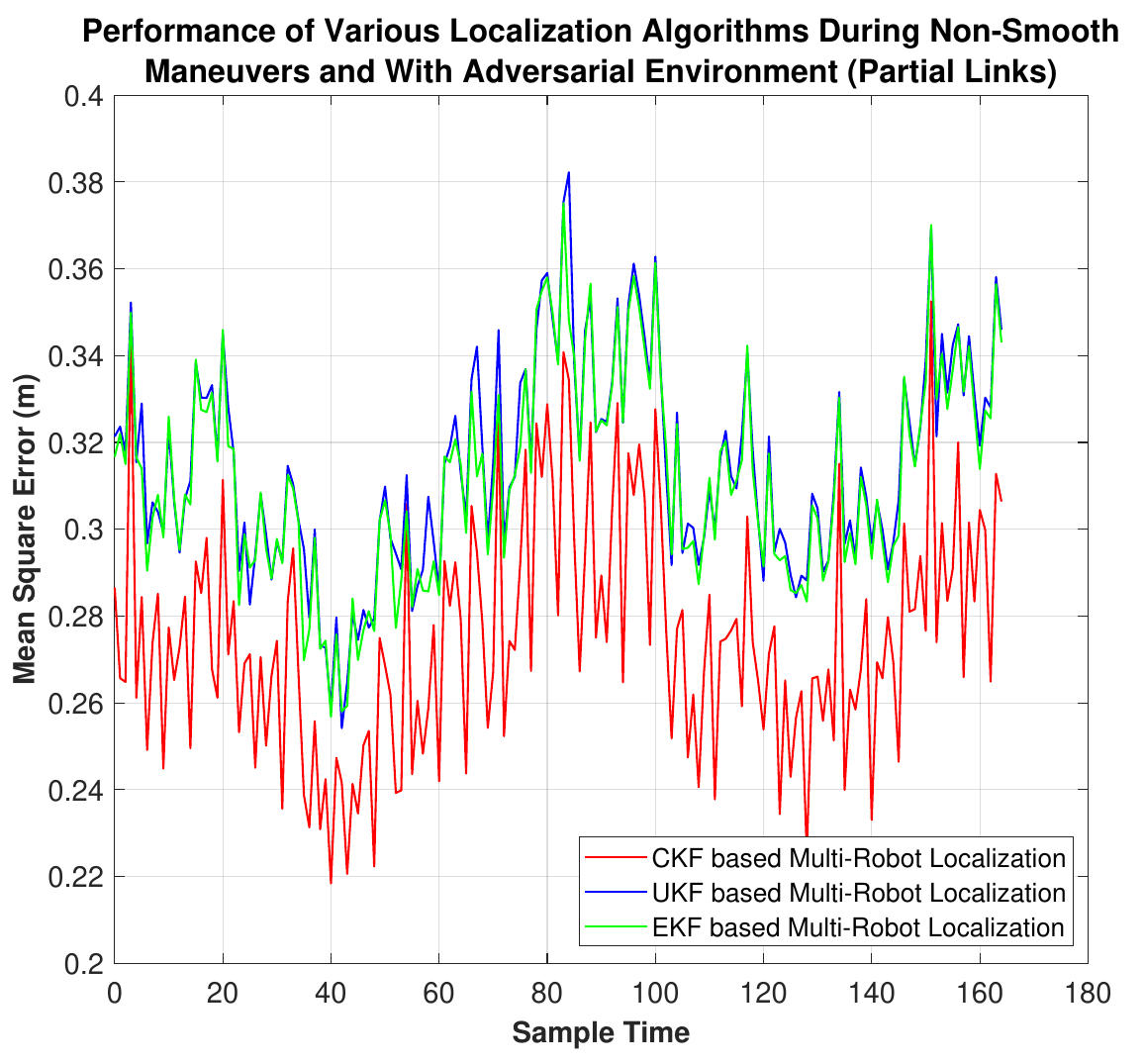}
        \caption{Adversarial environment with partial links affected.
}
        \label{fig:tiger2}
    \end{subfigure}
 ~ 
  \begin{subfigure}[b]{0.31\linewidth}
        \includegraphics[width=\textwidth]{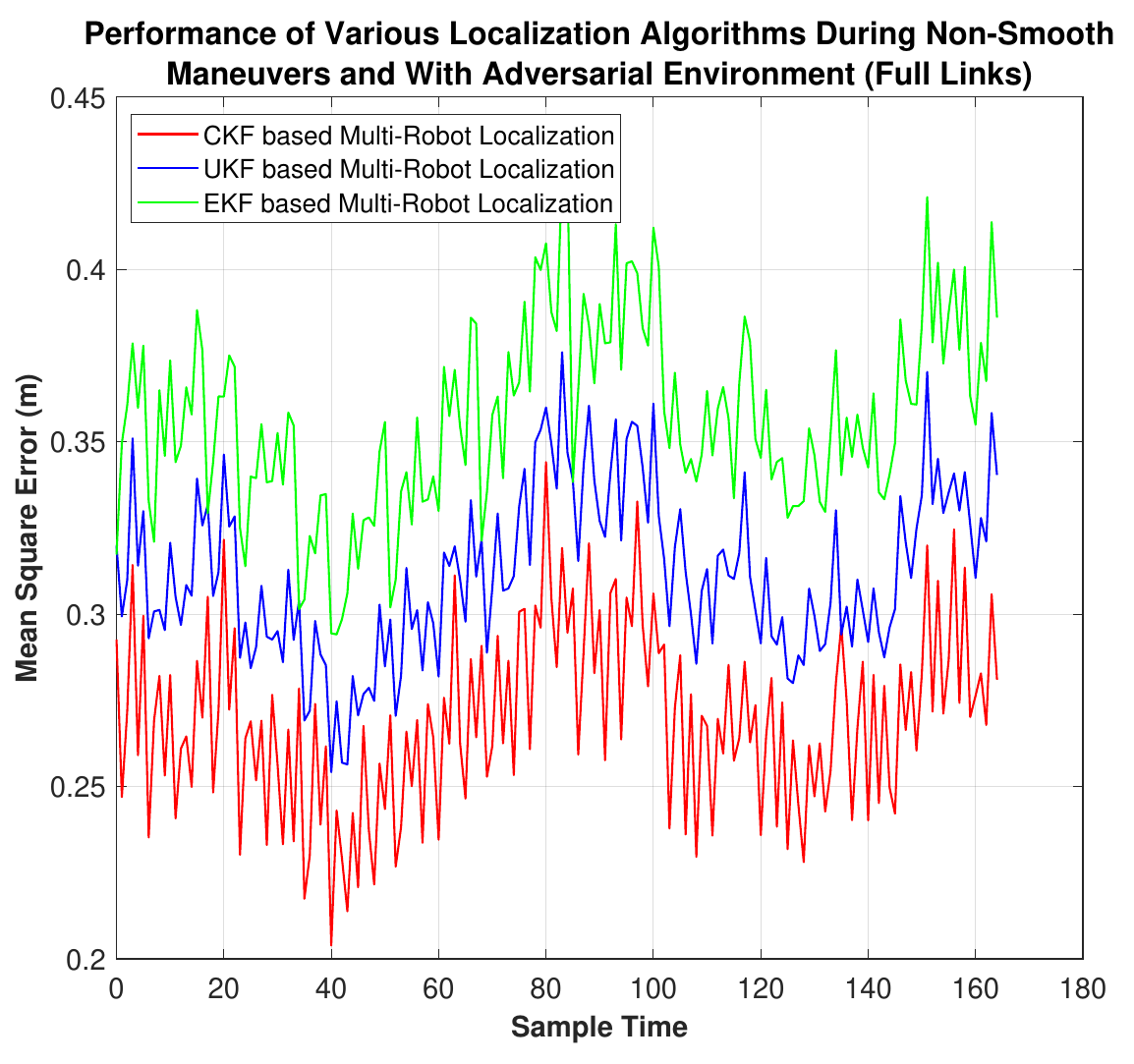}
        \caption{Adversarial environment with full links affected.
}
        \label{fig:tiger3}
    \end{subfigure}
 ~ 
    \caption{Localization performance of various localization algorithms during non-smooth maneuvers and without/with adversarial environment.}
  \label{fig:performance}
\end{figure*}

\begin{figure*}[t]
\captionsetup{justification=centering}
    \centering
        \begin{subfigure}[b]{0.235\textwidth}
        \includegraphics[width=\textwidth]{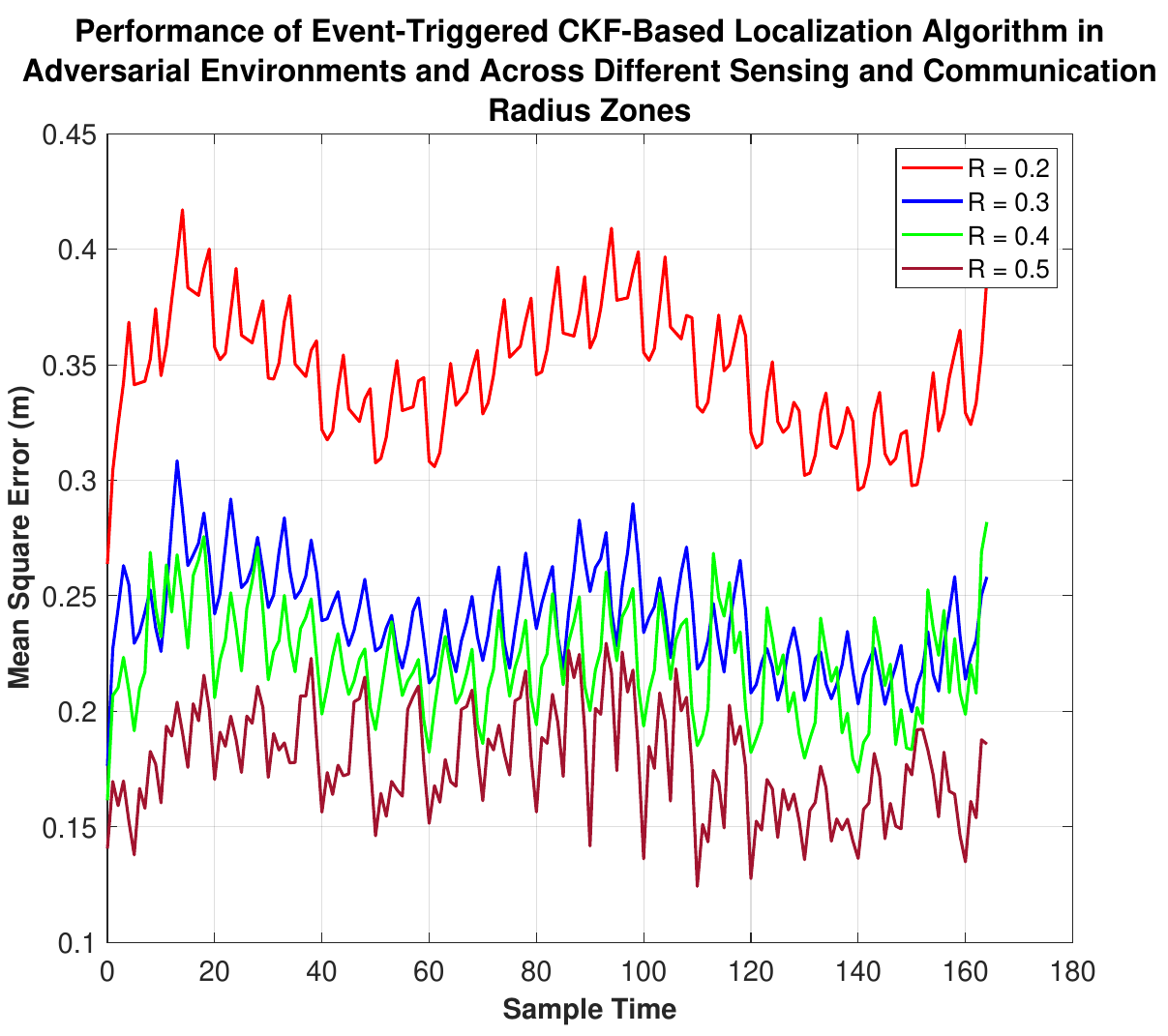}
        \caption{{\color{black} Localization results in both danger sensing and  communication zones with Different Radius }}
        \label{fig:gull1}
    \end{subfigure}
~ 
        \begin{subfigure}[b]{0.235\textwidth}
        \includegraphics[width=\textwidth]{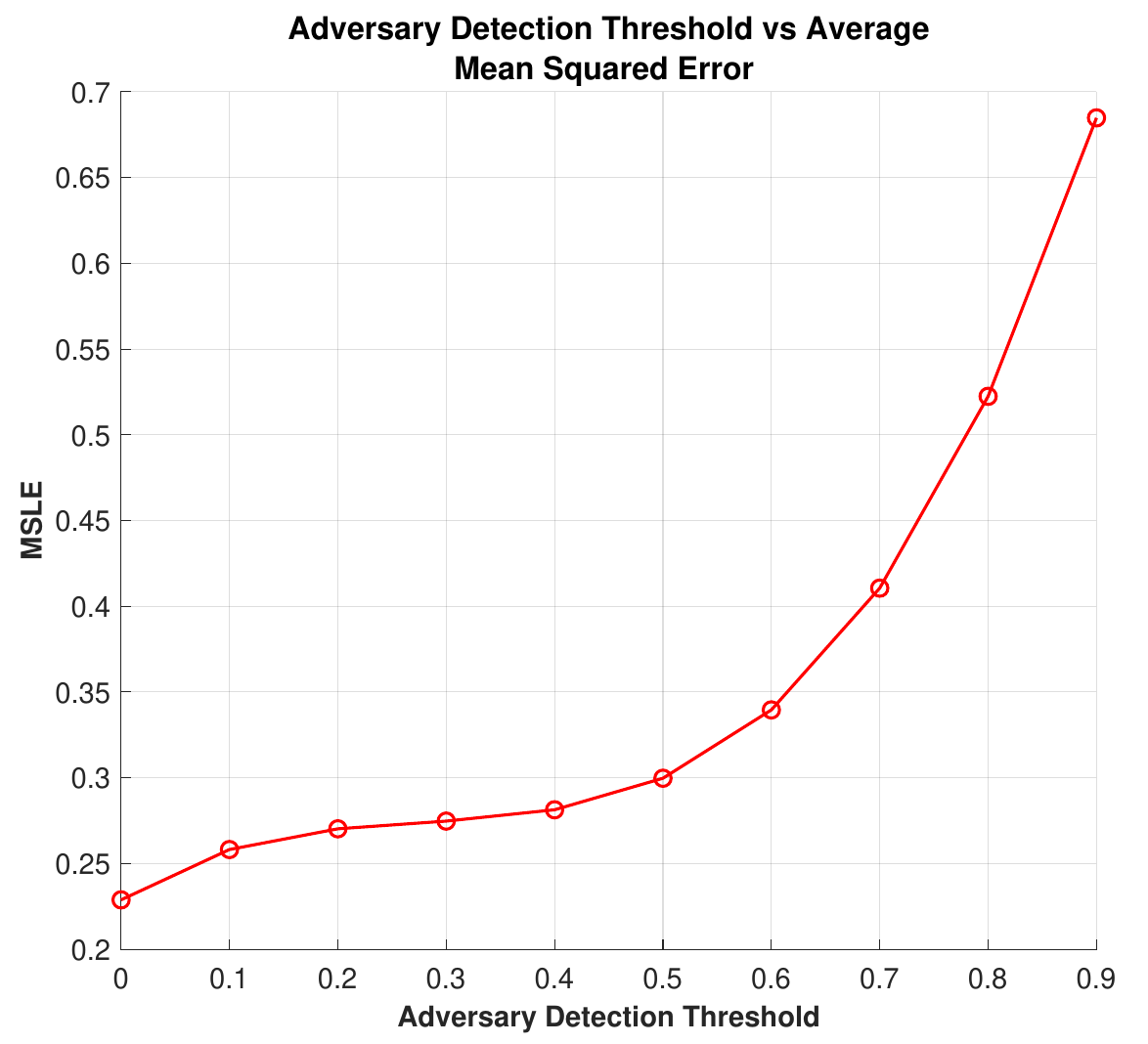}
        \caption{Varying the adversary detection threshold}
        \label{fig:gull2}
    \end{subfigure}  
    ~ 
        \begin{subfigure}[b]{0.235\textwidth}
        \includegraphics[width=\textwidth]{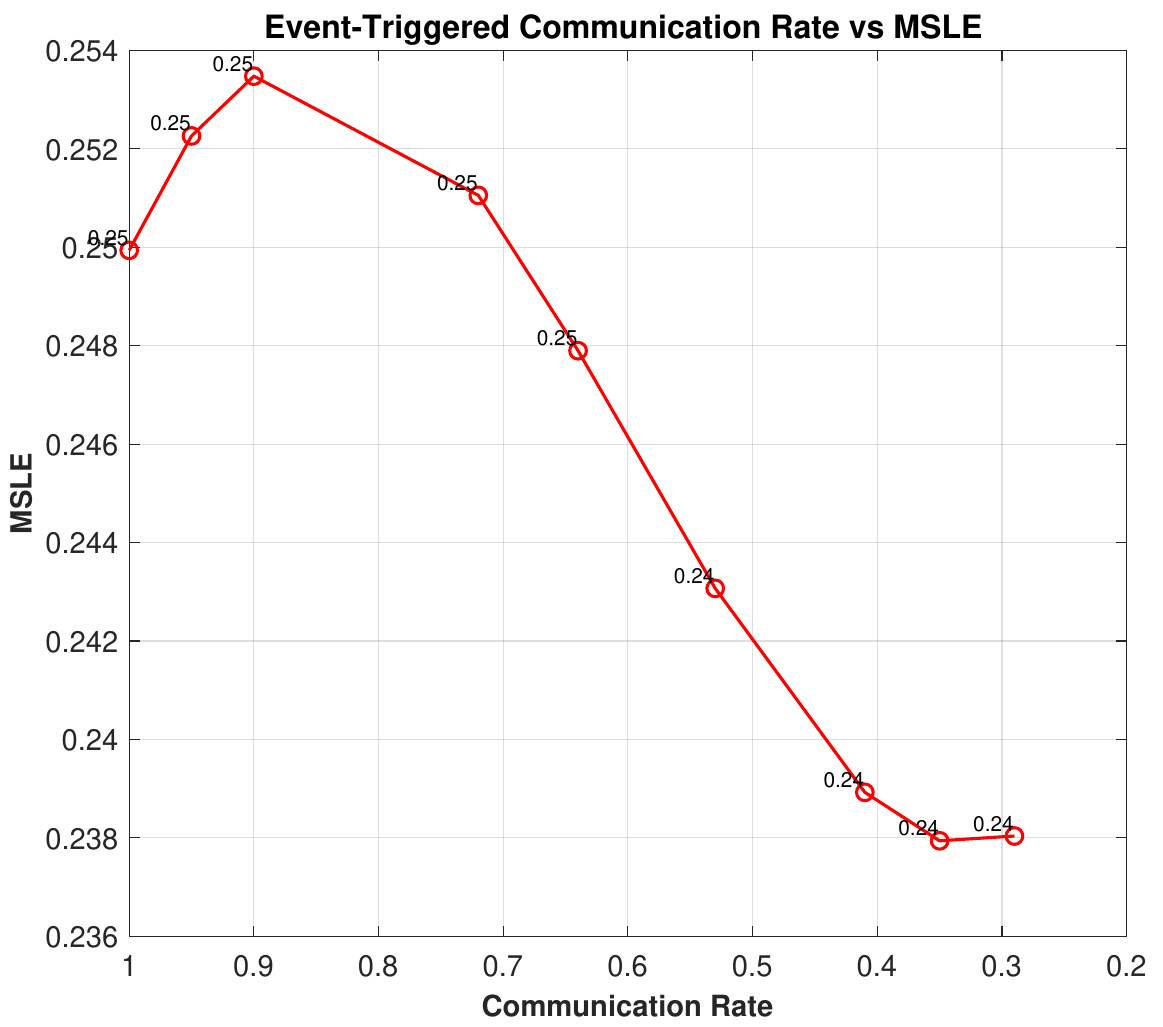}
        \caption{Increasing communication rate in the presence of an adversarial environment }
        \label{fig:gull3}
    \end{subfigure}  
     ~ 
        \begin{subfigure}[b]{0.235\textwidth}
        \includegraphics[width=\textwidth]{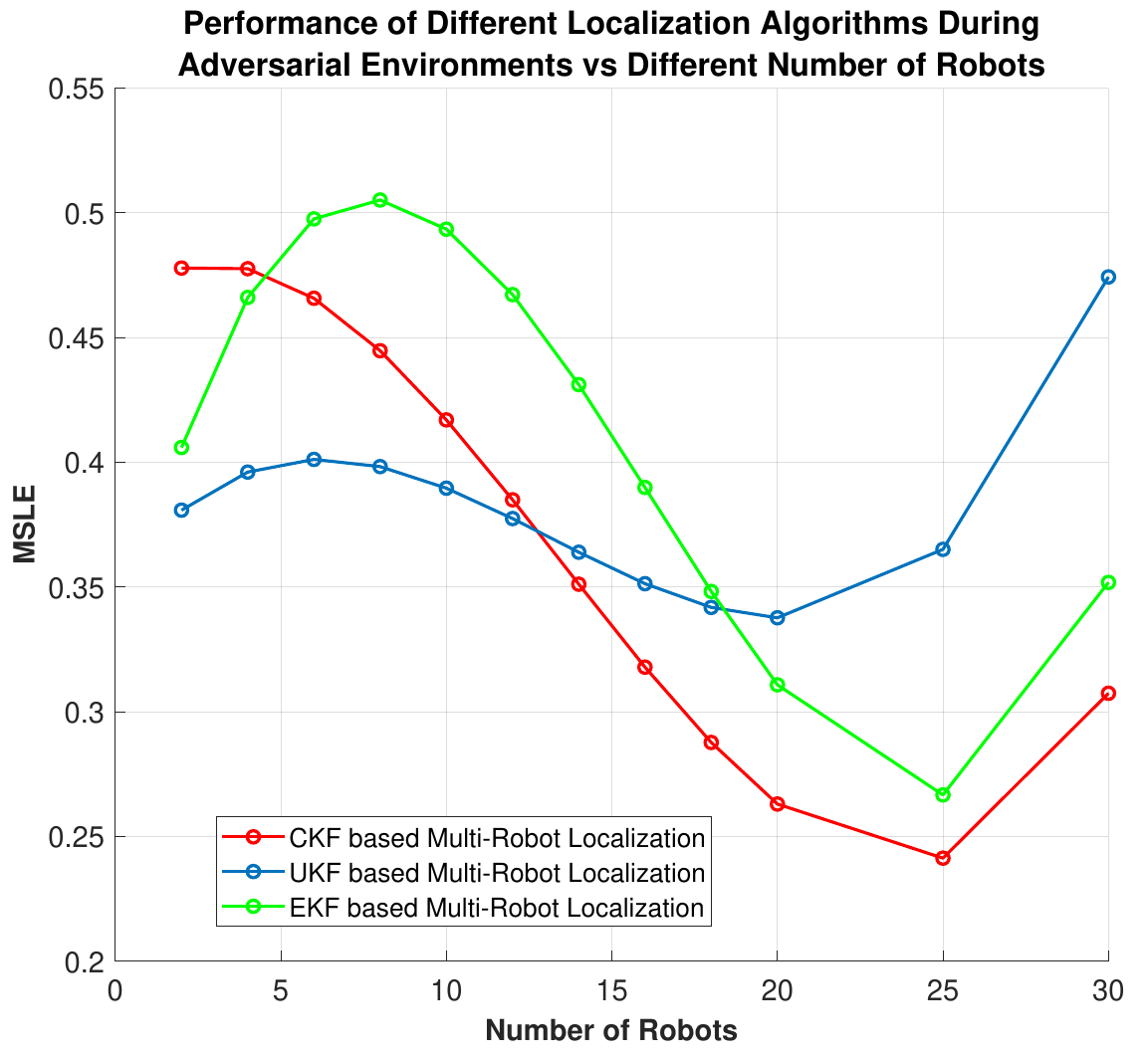}
        \caption{Increasing numbers of robots with adversaries }
        \label{fig:gull4}
    \end{subfigure}  
~ 
       \begin{subfigure}[b]{0.23\textwidth}
        \includegraphics[width=\textwidth]{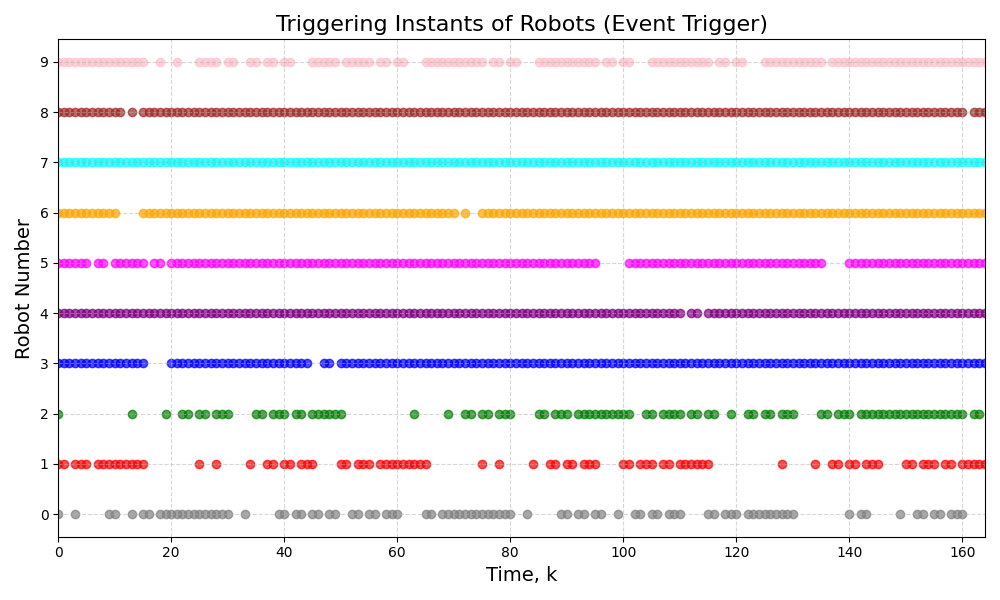}
        \caption{Transmission instances of robots with the
event-triggered mechanism under the adversarial radius
zones of 0.2 (Average Communication Rate = 0.3)}
        \label{fig:gull5}
    \end{subfigure}
~ 
       \begin{subfigure}[b]{0.23\textwidth}
        \includegraphics[width=\textwidth]{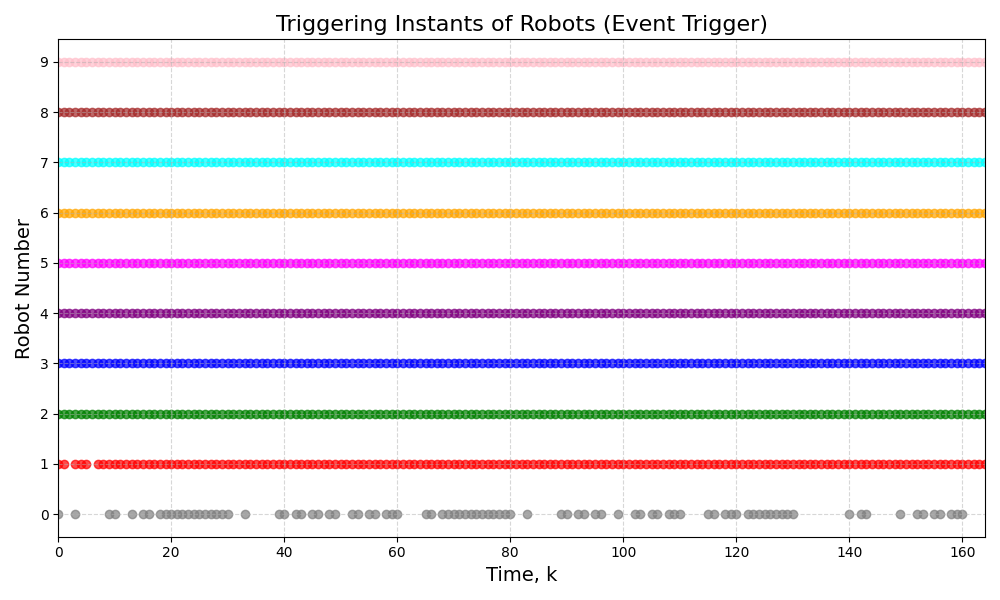}
        \caption{Transmission instances of robots with the
event-triggered mechanism under the adversarial radius
zones of 0.3 (Average Communication Rate = 0.75)}
        \label{fig:gull6}
    \end{subfigure}
~ 
       \begin{subfigure}[b]{0.23\textwidth}
        \includegraphics[width=\textwidth]{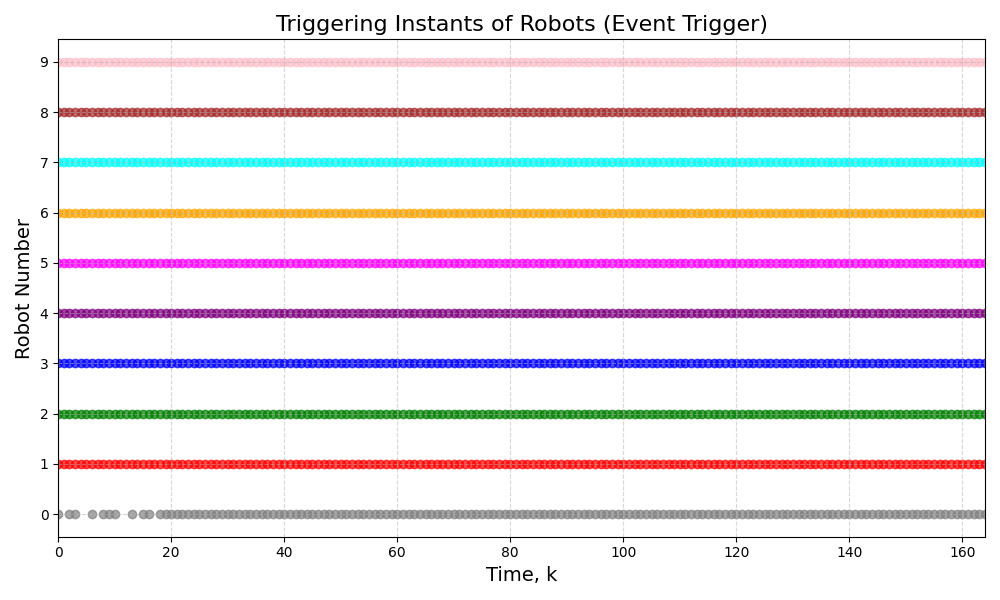}
        \caption{Transmission instances of robots with event-triggered mechanism under adversarial radius zones of 0.4 (Average Communication Rate = 0.89) }
        \label{fig:gull7}
    \end{subfigure}
~ 
       \begin{subfigure}[b]{0.23\textwidth}
        \includegraphics[width=\textwidth]{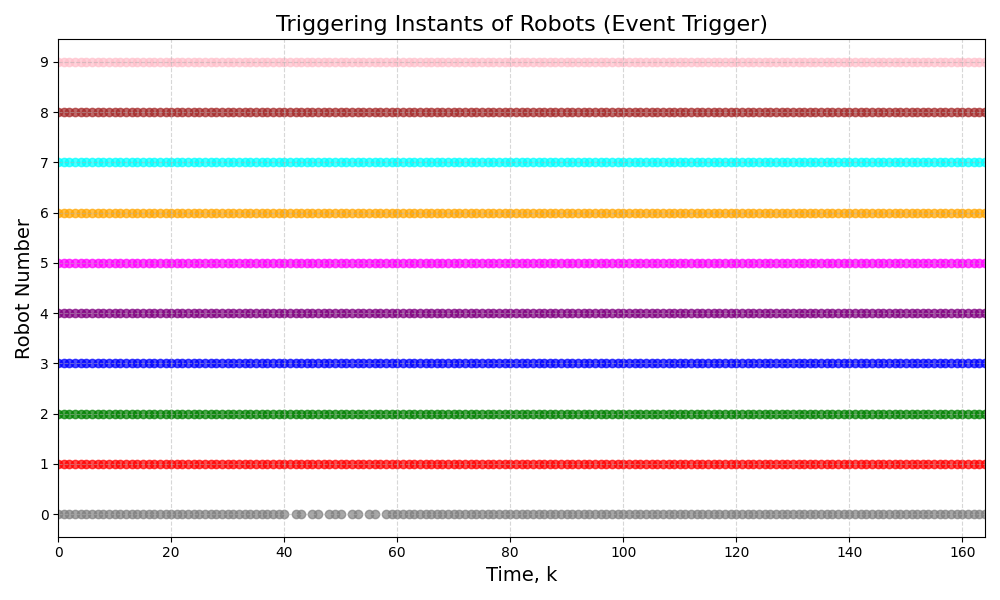}
        \caption{Transmission instances of robots with event-triggered mechanism under adversarial radius zones of 0.5
(Average Communication Rate = 0.93)}
        \label{fig:gull8}
    \end{subfigure}
    \caption{ Performance of CKF-based localization algorithm during
adversarial environment and different communication strategies.}
  \label{fig:events}
\end{figure*} 

\subsection{Impact of Adversaries on Cooperative Localization}
We provide a detailed comparison of the CKF, UKF, and EKF-based multi-robot localization algorithms under non-adversarial (Fig. 3a) and adversarial conditions (Fig. 3b), with particular emphasis on the influence of sensing and communication danger zones. In the non-adversarial environment, where no disruptions are present, the CKF-based algorithm consistently outperforms the UKF and EKF, achieving the lowest MSLE and exhibiting minimal fluctuations, highlighting its stability and accuracy during the arc-turn maneuver (see Fig. 4a). In Fig. 3b, the adversarial strategy involves sensing and communication danger zones, where relative sensor measurements are manipulated, and communication links are disrupted through jamming. 

Under these conditions, partial link failures cause noticeable performance degradation across all localization algorithms, with the UKF experiencing the most significant increase in error variability, while the CKF maintains a relatively stable and lower MSLE (see Fig. 4b). In the case of full-sensing and communication link failures, the impact of adversarial conditions is more severe, leading to higher MSLE for all localization algorithms. However, the CKF offers more resilient localization, with a smaller increase in error compared to the UKF and EKF (see Fig. 4c).  These results highlight the critical impact of sensing and communication danger zones on localization accuracy while also demonstrating the CKF's superior resilience and robustness in addressing these challenges. This makes the CKF-based algorithm particularly well-suited for dynamic MRS operating in adversarial environments.

\subsection{Impact of the Size of Adversarial Zones}
Here, we investigate the effects of the proposed event-triggered communication framework under varying adversarial zone radii and its impact on localization performance and communication rate. Fig. 5(a) shows the spatial distribution of robots performing an arc-turn maneuver, where both the danger sensing zones (red) and communication zones (blue) increase the radii by (0.2, 0.3, 0.4, and 0.5). As the radius increases, it leads to greater overlap with the robot network, requiring higher communication rates to preserve localization accuracy and maintain network resilience. Fig. 5(a) highlights the localization performance, demonstrating larger deviations in estimation errors as the radius of adversarial zones increases, thus reflecting higher communication demands to mitigate localization degradation. 

Figs. 5(e)-(h) analyzes the communication rates of robots in each scenario, demonstrating that smaller adversarial zones exploit the event-triggered mechanism by operating close to the detection threshold. This results in lower communication rates (e.g., 0.3 for a radius of 0.2) and a gradual degradation in localization quality. In contrast, larger zones trigger more frequent transmissions by raising average communication rates from 0.3 to 0.93, to mitigate the impact of increased adversarial zones. These results emphasize the framework's resilience against varying adversarial zone radii.

\subsection{Impact of Adversary Detection and Event-Triggers}
In this part, we assess the effect of adversary detection and event-triggered communication on localization performance in adversarial environments.
Fig. 5(b) illustrates the impact of the adversary detection threshold on the localization performance of the CKF-based framework. As the threshold increases, the MSLE steadily grows, indicating that a higher threshold allows more adversarial effects to remain undetected. This highlights the critical need to carefully tune the adversary detection threshold to mitigate the impact of adversarial environments while avoiding false positives, which can unnecessarily degrade performance.

Fig. 5(c) demonstrates the effect of the event-triggered communication mechanism under adversarial conditions. As the communication rate decreases, the MSLE improves slightly, showcasing the mechanism's effectiveness in reducing communication overhead without compromising localization accuracy. These results highlight the robustness of the event-based CKF localization algorithm, particularly in resource-constrained and adversarial scenarios.

\subsection{Scalability Analysis}
Fig. 5(d) presents the scalability performance of the CKF-based multi-robot localization algorithm in adversarial environments.{\color{black} As shown in Fig. 5(d), all methods perform worse with fewer robots due to the lack of relative observations, which limits the information available for accurate localization.} However, the performance improves with more robots, as the network becomes denser, enhancing observability and resilience to adversarial effects. The results demonstrate the robustness of CKF in maintaining high localization accuracy as the number of robots increases. For smaller team sizes (e.g., $N = 12$), CKF achieves a significantly lower MSLE compared to UKF and EKF, indicating its superior performance. As the team size grows, CKF continues to outperform both algorithms, with the MSLE consistently decreasing and stabilizing below $0.25$~m at $N = 25$. In contrast, the EKF-based approach shows notable degradation in performance, with the MSLE peaking around $N = 10$, followed by a marginal improvement. The UKF-based localization algorithm shows moderate performance but fails to scale effectively for larger team sizes ($N > 25$), where the MSLE increases substantially under adversarial conditions. The scalability and robustness of the CKF-based algorithm highlight its suitability for large-scale, cooperative robotic systems operating in adversarial and dynamic environments. 

Overall, the proposed localization framework demonstrates resilience under adversarial conditions by leveraging adaptive adversary detection and event-triggered communication strategies. These mechanisms enhance system performance, minimize communication overhead, and improve the robustness of MRS operating in adversarial environments.
However, {\color{black} this study assumes accurate modeling of sensor noise and adversarial behavior, which may not hold in all real-world scenarios. Future work will focus on learning-based threshold tuning, formal resilience guarantees under worst-case attacks, and extending the framework to real-world heterogeneous robot teams \cite{starks2023hero} as well as to 3D environments \cite{ghanta2024space}.}

\section{Conclusion}
In this work, we propose a decentralized cooperative localization framework based on the cubature Kalman Filter to enhance the resilience and accuracy of multi-robot systems operating in adversarial environments. The framework integrates an adaptive event-triggered communication strategy, which dynamically adjusts communication thresholds based on real-time sensing reliability and communication quality. This strategy improves the system’s resilience to adversaries by ensuring that only the most informative updates are transmitted, minimizing communication overhead while maintaining robust localization performance in the presence of sensor degradation and communication disruptions. Experimental results show that the proposed CKF-based algorithm outperforms traditional methods such as the EKF and UKF, particularly in adversarial conditions involving manipulated relative sensor measurements and communication jamming. The event-triggered strategy significantly enhances the system's ability to recover from communication and sensor failures, improving localization accuracy and system stability. Moreover, the framework demonstrates scalability, robustness, and efficiency, making it suitable for deployment in large-scale, dynamic, and resource-constrained MRS.

\section*{Acknowledgements}
Research was sponsored by the U.S. Army Research Laboratory and was accomplished under Cooperative Agreement Number W911NF-17-2-0181 (DCIST-CRA). 
The views and conclusions contained in this document are those of the authors and should not be interpreted as representing the official policies, either expressed or implied, of the Army Research Laboratory or the U.S. Government. The U.S. Government is authorized to reproduce and distribute reprints for Government purposes notwithstanding any copyright notation herein.

\end{document}